\documentclass[10pt,twocolumn,letterpaper]{article}

\usepackage[pagenumbers]{cvpr}

\usepackage{flushend}

\usepackage{graphicx}
\usepackage{amsmath}
\usepackage{amssymb}

\usepackage[pagebackref,breaklinks,colorlinks]{hyperref}

\usepackage{amsmath,amssymb,amsthm}

\usepackage{pifont}
\newcommand{\cmark}{\ding{51}}%
\newcommand{\xmark}{\ding{55}}%

\usepackage{xspace}
\makeatletter
\DeclareRobustCommand\onedot{\futurelet\@let@token\@onedot}
\def\@onedot{\ifx\@let@token.\else.\null\fi\xspace}

\def\eg{\emph{e.g}\onedot} 
\def\ie{\emph{i.e}\onedot}

\def\wrt{w.r.t\onedot} 
 \def\wolog{w.l.o.g\onedot}
\def\etal{\emph{et al}\onedot}
\makeatother

\renewcommand{\phi}{\varphi}
\renewcommand{\epsilon}{\varepsilon}

\usepackage{soul}
\renewcommand{\underline}{\ul}

\usepackage{bbm}

\providecommand{\norm}[1]{\ensuremath{\left\lVert#1\right\rVert}}

  \providecommand{\R}{\mathbb{R}} 
  
  \DeclareMathOperator*{\argmin}{arg\,min}

  \providecommand{\0}{\mathbf{0}}

  \providecommand{\ii}{\mathbf{i}}

  \providecommand{\pp}{\mathbf{p}}

  \renewcommand{\ss}{\mathbf{s}}

  \providecommand{\vv}{\mathbf{v}}
  
  \providecommand{\xx}{\mathbf{x}}

  \providecommand{\mC}{\mathbf{C}}

  \providecommand{\mJ}{\mathbf{J}}
  \providecommand{\mK}{\mathbf{K}}
  
  \providecommand{\mM}{\mathbf{M}}

  \providecommand{\mS}{\mathbf{S}}

  \providecommand{\cE}{\mathcal{E}}
  \providecommand{\cF}{\mathcal{F}}

  \providecommand{\cL}{\mathcal{L}}
  \providecommand{\cM}{\mathcal{M}}

  \providecommand{\cP}{\mathcal{P}}

  \providecommand{\cS}{\mathcal{S}}
  \providecommand{\cT}{\mathcal{T}}
  
  \providecommand{\cV}{\mathcal{V}}

\newtheorem{proposition}{Proposition}
\newtheorem{assumption}{Assumption}

\usepackage[capitalize]{cleveref}
\crefname{assumption}{Assumption}{Assumptions}
\crefname{example}{Example}{Examples}
\crefname{claim}{Claim}{Claims}

\usepackage[acronym]{glossaries}
\newacronym{gplvm}{GP-LVM}{Gaussian Process Latent Variable Model}
\newacronym{sft}{SfT}{Shape from Template}
\newacronym[firstplural=Gaussian Processes (GPs)]{gp}{GP}{Gaussian Process}
\newacronym{desurt}{DeSurT}{Deformable Surface Tracking}
\newacronym{tso}{TSO}{Tracking Surface with Occlusion}
\newacronym{tds}{TDS}{Texture-less Deformable Surfaces}
\newacronym{relu}{ReLU}{Rectified Linear Unit}

\usepackage{algorithm}
\usepackage{algorithmic}

\usepackage[table]{xcolor}
\usepackage{booktabs, multirow}
\usepackage{makecell}

\usepackage{subcaption}

\newcommand\blfootnote[1]{%
  \begingroup
  \renewcommand\thefootnote{}\footnote{#1}%
  \addtocounter{footnote}{-1}%
  \endgroup
}

\begin{document}

\title{Deformable Surface Reconstruction via Riemannian Metric Preservation}

\author{Oriol Barbany$^\dagger$\hspace{3ex}Adrià Colomé\hspace{3ex}Carme Torras \\[1ex]
Institut de Robòtica i Informàtica Industrial (CSIC-UPC), Barcelona, Spain
}
\maketitle
\blfootnote{$\dagger$ Corresponding author: \texttt{obarbany@iri.upc.edu}.}
\begin{figure*}
    \centering
    \includegraphics[width=.9\linewidth]{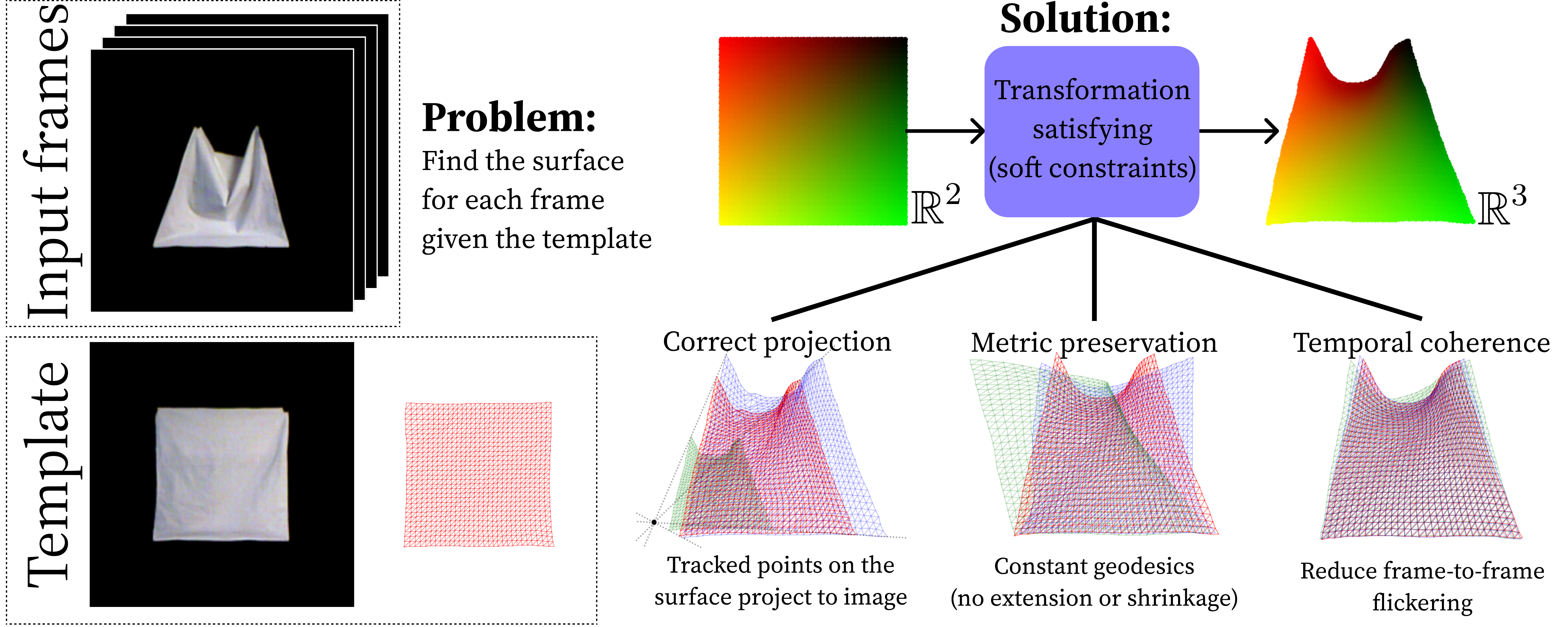}
    \caption{\textbf{Summary of the presented method.} Given a template consisting of a surface representation (in this case, a mesh) and an image showing such a surface, we want to recover deformations of the same surface as it appears in the input frames. We propose describing the surface by a parametric equation learned by an iterative process. A parametric equation maps a point in $\R^2$, uniquely defining a point on the surface, to its spatial coordinates in $\R^3$ (in the example shown above, this mapping preserves the colors, which we add for easier visualization). To find such a parametric equation, we impose three soft constraints that alone are insufficient but, when combined, allow us to recover the surface. We illustrate the ambiguities of each condition by showing feasible solutions that are either correct (in red) or incorrect (in blue and green). The constraints and ambiguities are: (1) The surface correctly projects to the input image, a condition satisfied by any surface obtained after arbitrarily moving the tracked points along the line of sight. (2) The geodesic distances are preserved, a condition satisfied for any isometric transformation of the template surface. (3) The surface does not vary too much in consecutive frames, which is compatible with all surfaces whose distance between the surface from the previous frames is small.}
    \label{fig:graphical_abstract}
\end{figure*}

\begin{abstract}
    Estimating the pose of an object from a monocular image is an inverse problem fundamental in computer vision.
    The ill-posed nature of this problem requires incorporating deformation priors to solve it. In practice, many materials do not perceptibly shrink or extend when manipulated, constituting a powerful and well-known prior. Mathematically, this translates to the preservation of the Riemannian metric.
    Neural networks offer the perfect playground to solve the surface reconstruction problem as they can approximate surfaces with arbitrary precision and allow the computation of differential geometry quantities. This paper presents an approach to inferring continuous deformable surfaces from a sequence of images,
    which is benchmarked against several techniques and obtains state-of-the-art performance without the need for offline training.
\end{abstract}


\section{Introduction}
\label{sec:intro}

In this paper, we tackle the problem of inferring the shape of a generic non-rigid object from a sequence of monocular images given a 3D template of the manipulated object. This problem is known as \gls*{sft} \cite{shape_from_template} and appears in many areas like entertainment \cite{animation_application}, medicine \cite{medical_application}, and robotic manipulation \cite{robotic_manipulation_application}. Our motivation stems from this last field, in which perceiving the state of deformable objects is one of the bottlenecks often pinpointed as hindering their manipulation by robots \cite{sanchez}. In this area, estimating the cloth state is crucial to simulate its evolution under manipulation \cite{inextensible} and apply model-predictive control \cite{luque} to guide the robots.

The \gls*{sft} problem is inherently ill-posed as it permits infinitely many solutions leading to accurate projection to the input 2D images \cite{TretschkNonRigidSurvey}
and requires the incorporation of deformation priors. One of the most common hypotheses is that the object is obtained from an isometric transformation of the template \cite{shape_from_template,analytical_sol_proofs,Parashar_2015_ICCV,chhatkuli17,brunet2010,reconstruct_sharply_folding,isowarp}, meaning that it cannot extend or shrink. This condition is often not restrictive in practice as many deformable objects made of materials such as cloth and paper are nearly inextensible \cite{linear_local_model}.
Although this constitutes a powerful and widely applicable prior, it defines a differential equation that, in most cases, has to be approximated to make the problem computationally tractable.

In this work, we propose to encode the parametric equations of the object of interest in the weights of a neural network. Parametric equations map 2D coordinates representing a point on a surface to its 3D coordinates. Therefore, they provide a continuous surface contrasting with the discrete representations used in previous works. To find the parameters of the parametric equations, we perform an iterative procedure enforcing three soft constraints accounting for a correct surface projection, the preservation of the surface metric (equivalent to the isometry constraint), and temporal consistency with the previous surface. We depict the presented method in \cref{fig:graphical_abstract}.

The isometry assumption is equivalent to the hypothesis that the geodesics on the surface may not change their length over time, which we can enforce by preserving the Riemannian metric. Given the parametric equations, we can analytically compute differential geometric quantities such as the metric tensor. Unlike previous methods struggling to optimize non-convex objectives and considering differential equations, using a neural network and stochastic optimizers allows for estimating the parameters of the surface in a relatively small amount of time. While there exist methods using explicit neural representations that have been used for single-view reconstruction \cite{atlasnet,parametric_surface_metric}, this is the first method to use them without offline training and in conjunction with the popular isometry constraint.

We evaluate the proposed method against five approaches and a baseline introduced in this paper on the publicly available datasets \gls*{desurt} and \gls*{tds}. The quantitative results show that our method achieves the lowest or the second-lowest mean tracking error and standard deviation in all the tested sequences, showing its effectiveness and robustness. Compared to other optimization-based approaches considered in this work, our solution is the fastest, which confirms that one can analytically enforce isometry and keep the computation feasible using the proposed framework.

Our main contributions are:
\begin{itemize}
	\item We combine the powerful and widely applicable assumption of metric preservation with the representation power of neural networks. In particular, we learn a surface parametrization \cite{atlasnet}, that allows \textbf{representing continuous surfaces} rather than discretized representations (\eg, point clouds, voxels, or meshes).
	\item We advocate for imposing \textbf{metric preservation of the surface as a soft constraint}, which accounts for the fact that physical quantities are not exactly preserved \cite{alet2021noether,tailoring,closed_form_solution}.
        \item We learn the parameters of the neural explicit surface during inference and \textbf{without an offline training process}. Therefore, the method requires neither a dataset nor fine-tuning to new sequences, materials, or template representations.
\end{itemize}

The rest of this paper is structured as follows. In \Cref{sec:related_work}, we review previous approaches to the \gls*{sft} problem. Then, we introduce the proposed method in \Cref{sec:method} and show its qualitative and quantitative performance in \Cref{sec:experiments}. Finally, we conclude the paper in \Cref{sec:conclusions}.

\section{Related Work}
\label{sec:related_work}

In this section, we review several approaches to the \gls*{sft} problem and group them under three umbrellas: the methods that use hand-crafted constraints to define the manifold of surfaces, those that infer the deformation models from data, and the approaches that combine analytical and data-driven models. \cref{tab:related_work} presents a summary of the discussed papers.

\begin{table*}[t]
    \resizebox{%
      \ifdim\width>\linewidth
        \linewidth
      \else
        \width
      \fi
    }{!}{%
    \centering
    \rowcolors{2}{gray!10}{}
    \begin{tabular}{cccccc}
        \toprule
        Method & \makecell{Works with\\non-planar\\template} & \makecell{Models\\sharp\\folds} & \makecell{Does not\\need\\dataset} & \makecell{Temporal\\coherence} & \makecell{Variable\\mesh\\resolution}\\
        \hline
        \multicolumn{6}{>{\cellcolor[gray]{.8}}l}{\textbf{Hand-crafted constraints (\cref{sec:handcrafted})}} \\
        \hline
        Iterative isometric surfaces \cite{brunet2010} & \xmark & \cmark & \cmark & \xmark & \xmark \\
        Closed-form isometric surfaces (I) \cite{shape_from_template,analytical_sol_proofs} & \cmark & \xmark & \cmark & \xmark & \xmark \\
        Closed-form isometric surfaces (II) \cite{Parashar_2015_ICCV,chhatkuli17,isowarp} & \cmark & \cmark & \cmark & \xmark & \xmark \\
        Iterative constant Euclid. \cite{surface_deformation_models} & \xmark & \xmark & \xmark & \cmark & \xmark \\
        Closed-form constant Euclid. \cite{closed_form_solution} & \xmark & \xmark & \xmark & \xmark & \xmark \\
        Inextensible surfaces \cite{reconstruct_sharply_folding} & \xmark & \cmark & \cmark & \xmark & \xmark \\
        Dense registration  \cite{dense_image_registration} & \xmark & \cmark & \cmark & \cmark & \xmark \\
        Laplacian meshes \cite{laplacian_mesh,surface_track_graph_matching} & \cmark & \cmark & \cmark & \cmark & \xmark \\
        Edge orientation changes \cite{convex_optimization} & \cmark & \cmark & \cmark & \cmark & \xmark \\
        Vertex coordinate changes \cite{temporal_coherence} & \cmark & \cmark & \cmark & \cmark & \xmark \\
        \hline
        \multicolumn{6}{>{\cellcolor[gray]{.8}}l}{\textbf{Data-based constraints (\cref{sec:databased})}} \\
        \hline
        Latent space \cite{linear_local_model,local_deformation_model,constrained_lvm} & \cmark & \cmark & \xmark & \xmark & \xmark \\
        Neural network: Image to vertices/depth \cite{texture-less,deepsft,texture_generic_deepsft,pumarola2018geometry} & \cmark & \cmark & \xmark & \xmark & \xmark \\
        Surface parametrization \cite{atlasnet,parametric_surface_metric} & \cmark & \cmark & \xmark & \xmark & \cmark \\
        \hline
        \multicolumn{6}{>{\cellcolor[gray]{.8}}l}{\textbf{Hybrid approaches (\cref{sec:hybrid})}} \\
        \hline
        GP constant Euclid. \cite{implicitly_constrained_gp} & \xmark & \xmark & \xmark & \xmark & \xmark \\
        Differentiable physics simulator and renderer \cite{phi_sft} & \cmark & \cmark & \cmark & \cmark & \xmark \\
        \midrule
        Ours & \cmark & \cmark & \cmark & \cmark & \cmark \\
        \bottomrule
    \end{tabular}
    }
    \caption{\textbf{Summary of related work}. Our method is the first to satisfy all the listed desirable properties for the reconstruction of deformable surfaces.}
    \label{tab:related_work}
\end{table*}

\subsection{Hand-crafted constraints}
\label{sec:handcrafted}
These methods leverage explicit properties of the tracked surface to determine the manifold of possible shapes. Given that constraints are hand-crafted, we can easily modify them, and their effects are well understood. However, some constraints rarely describe the complex and non-linear physics that real surfaces exhibit for large deformations \cite{closed_form_solution}, which would require designing complex objectives that need specific knowledge of the surface \cite{reconstruct_sharply_folding,linear_local_model}.

We can reconstruct a surface unambiguously if its Riemannian metric is assumed to be preserved \cite[Theorem 1]{analytical_sol_proofs}. Brunet \etal \cite{brunet2010} exploited this property proposing a method that requires planar templates. Several works \cite{shape_from_template,analytical_sol_proofs,Parashar_2015_ICCV,chhatkuli17,isowarp} enforced metric preservation by relying on a differentiable warp between the template and the images in a sequence. Such warp is undefined in occluded areas, and its required differentiability is incompatible with sharp folds.

Given the difficulty of enforcing surface metric preservation, some works proposed to use relaxations using the Euclidean norm between vertices. Enforcing equality constraints \cite{closed_form_solution,surface_deformation_models} is incompatible with sharp folds, and inequality constraints \cite{reconstruct_sharply_folding,linear_local_model,local_deformation_model} are prone to vertex collapsing. To prevent the latter, some works use the maximum depth heuristic, which does not consider surface properties \cite{analytical_sol_proofs}.

Laplacian meshes provide a framework for constraining the problem by reducing the number of free parameters \cite{laplacian_mesh,surface_track_graph_matching}. However, this approach loses resolution on the edges of the surface and yields 3D shape estimates that re-project to the image but are not necessarily accurate.

Another typical assumption is that the surface does not vary too much in neighboring frames, which is a mild assumption for high enough image rates and is known as the short-baseline case in the \gls*{sft} literature. Enforcing temporal smoothness helps recover surfaces with severe deformations and reduces jitter. We can reduce frame-to-frame flickering by minimizing the difference with the previous time step \cite{temporal_coherence}, a window of past time steps \cite{reconstruct_sharply_folding}, the second derivative of the surface parameters \cite{surface_deformation_models}, or the change of edge orientation along time \cite{convex_optimization}. Other works leverage temporal smoothness to obtain good initializations to the solution \cite{dense_image_registration,laplacian_mesh,phi_sft}.

\subsection{Data-based constraints}
\label{sec:databased}

Data-based approaches learn the manifold of plausible surfaces using statistical learning techniques. The drawback of this class of methods is that they require a dataset of surface configurations, which ideally should contain all the possible deformations and be representative enough of the dynamics of the material. Listing all deformations is impossible for deformable objects like a cloth, which we can arrange in infinitely many ways. An additional limitation is that, for these methods, we may need a different dataset for each type of material, surface, lighting conditions, and mesh resolution.

We can obtain the manifold of possible shapes using dimensionality reduction techniques. Previous works have considered PCA \cite{linear_local_model,closed_form_solution,surface_deformation_models}, sparse \glspl*{gplvm} \cite{local_deformation_model} and constrained latent variable models \cite{constrained_lvm}. The dimensionality of the latent space may depend on the material \cite{surface_deformation_models}, and the models can yield extensible surfaces even if the dataset only consists of inextensible surfaces \cite{surface_deformation_models}.

An increasingly popular approach is to use neural networks to either predict the surface vertices from an image \cite{texture-less,pumarola2018geometry} or predict the depth map and use it to infer the coordinates of the points in the tracked surface \cite{deepsft,texture_generic_deepsft}. In general, such methods require learning on large datasets, and in some cases, they only allow prohibitively small mesh sizes, \eg, $10\times10$ \cite{pumarola2018geometry}. The methods relying on depth estimation only recover the visible points in an image and require
post-processing the resulting point cloud with As-Rigid-As-Possible regularization \cite{arap}.

An interesting approach is to learn the parametric equations of the surface depicted in the input image \cite{atlasnet,parametric_surface_metric}, which can generate continuous surfaces. These approaches are trained on pairs of images and their point clouds and hence are also object-specific.

\subsection{Hybrid approaches}
\label{sec:hybrid}

A promising research direction highlighted in the context of perception for robotic cloth manipulation \cite{review_kragic}  is to combine analytical and data-driven models. Salzmann \etal \cite{implicitly_constrained_gp} used \glspl*{gp} implicitly satisfying a set of quadratic equality constraints, which are overly simplistic for sharply folding materials like clothes \cite{constrained_lvm}. This method also requires that all the training examples satisfy all the constraints.

Kairanda \etal \cite{phi_sft} proposed integrating a differentiable physics simulator and renderer. The drawbacks of this method are that it imposes hard constraints through the physics simulator, requires 16-24 hours on a GPU to process an image sequence, and uses a fixed resolution of around 300 vertices, which prevents capturing fine wrinkles \cite{phi_sft}.

This work falls into this category, as we also propose leveraging simulation equations devised to constrain the dynamics of meshes and use them to solve the inverse problem. Concretely, we adapt the inextensible model for the manipulation of textiles \cite{inextensible}, which follows the isometry assumption.

\section{Methodology}
\label{sec:method}

In this section, we first introduce the problem notation and some preliminaries. Then, we integrate the relaxation of the isometry constraint developed for a cloth simulator \cite{inextensible} into a classical optimization scheme. This method is used as a baseline and dubbed as \textsc{Classical}. Finally, we introduce our method.

\subsection{Preliminaries}

The first fundamental form of a surface $\cS$ allows measuring lengths of curves, angles of tangent vectors, and areas of regions on it \cite{differential_geometry_book}. Using a parametrization $\phi: \cP \subset \R^2 \to \R^3$, the Riemannian metric of $\cS$ is then uniquely defined by the metric tensor $\mJ_{\phi}^ T\mJ_{\phi}$, where $\mJ_{\phi}$ is the Jacobian matrix of the map $\phi$. We focus on modeling surfaces in such a way that we preserve the Riemannian metric. That is, at all times, the length of any curve inside the surface remains constant.

\begin{assumption}
	The sequence of monocular images used as input to the \gls*{sft} problem is obtained with a calibrated camera with known intrinsic parameters.
	\label{ass:camera}
\end{assumption}

Following the pinhole camera model, given a point $\ss = (x, y, z) \in \cS$, which \wolog we assume to be expressed in the camera referential, we can compute the position $(u, v)$ in the image captured by the camera as follows:
\begin{equation}
	d
	\begin{bmatrix}
		u \\
		v \\
		1 \\
	\end{bmatrix} =
	\mK \begin{bmatrix}
		x \\
		y \\
		z
	\end{bmatrix}\,,
	\label{eq:mesh2img}
\end{equation}
where $\mK$ is the intrinsic matrix, known according to \cref{ass:camera}, a common assumption in \gls*{sft}, and $d$ is the depth along the line of sight. In the following, let $\Pi(\ss):=(u, v)$.

\textbf{Problem: }Given a sequence of images $\{I^{(t)}\}_{t \in [T]}$, where $[T]:=\{1, \dots, T\}$, and a template surface $\cT$, estimate the surface $\cS^{(t)}$ at any time $t$.

\begin{figure}[t]
	\centering
	\includegraphics[width=\columnwidth]{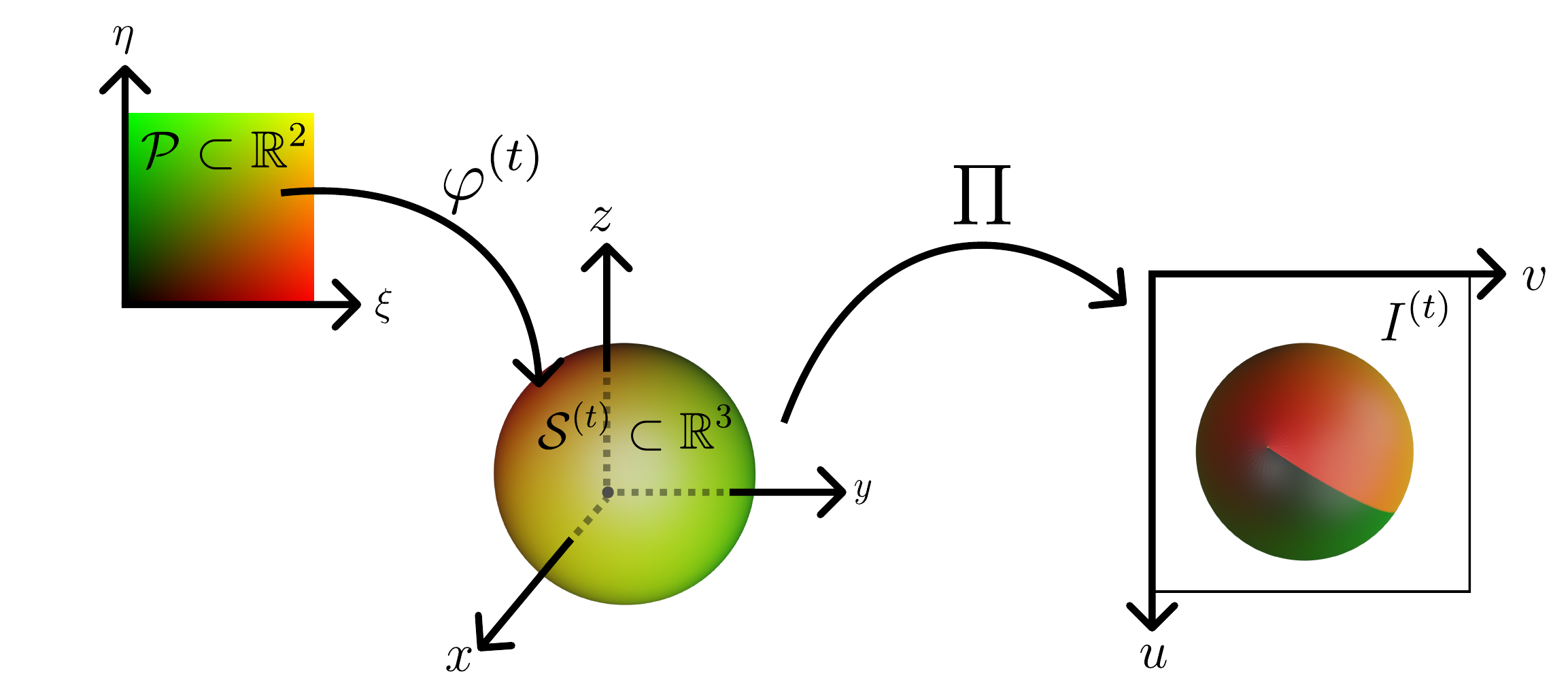}
	\caption{\textbf{Problem notation.} Scheme illustrating the notation used in this manuscript. The parametric equations $\phi^{(t)}$ map the input domain $\cP\in \R^2$, which in this case is the interior of a unit square, to the 3D coordinates of $\cS^{(t)}$, the surface at time $t$. Then, an image $I^{(t)}$ is obtained by projecting the surface using $\Pi$. We add colors to the input of the parametrization $\cP$ and preserve them after each transformation.}
	\label{fig:notation}
\end{figure}

\subsection{Classical approach}
\label{sec:classical}

Let $\cM_{\cS}:=(\cV,\cE,\cF)$ be a triangle mesh that approximates the smooth surface $\cS$ with $n$ vertices $\cV:=\{\vv_i\}_{i\in[n_v]}$, edges $\cE\subseteq \cV^2$ and triangular faces $\cF \subseteq \cV^{3}$.

A typical strategy to formulate the \gls*{sft} problem is to proceed in two steps. First, in the registration step, we find correspondences between a known shape and an image showing an unknown deformation with a feature-matching algorithm providing
\begin{align}
	M^{(t)} := \{(\ss, \ii) : \ss \in \cS^{(t)}, \ii = \Pi(\ss) \in I^{(t)}\}\,.
	\label{eq:matchings}
\end{align}

Then, in the reconstruction step, we obtain the deformed shape, which requires inferring the depth of the image points obtained in registration.

Salzmann \etal \cite{ambiguities} showed that the projection constraints given by the registration algorithm can be expressed in a linear system of equations of the form
\begin{align}
	\mM^{(t)}\xx^{(t)} = \0\,,
	\label{eq:Mx0}
\end{align}
where $\mM^{(t)}\in \R^{2|M^{(t)}| \times 3 n_v}$ expresses the correspondences and $\xx^{(t)}:=\text{vec}\left(\begin{bmatrix}\vv_1^{(t)} & \cdots & \vv_N^{(t)}\end{bmatrix}\right)$.

To compute $\mM^{(t)}$, the tracked points on the surface are expressed using barycentric coordinates. Recall that, given a point $\pp$ belonging to facet $f$ of $\cM_{\mS}$, we have that $\pp=\sum_{i\in[3]}b_i\vv_{f,i}$, where $\sum_{i\in [3]}b_i=1$ and $\{b_i\}_{i\in [3]}$ are the barycentric coordinates of $\pp$.

Every solution to \cref{eq:Mx0} re-projects correctly to the image, but the vertex positions are not guaranteed to correspond with the true ones because of the depth ambiguity. Given enough correspondences, the constraints given by the registration step are enough and do not require knowing the depth \cite{ambiguities}. However, the imperfection of image correspondences in real scenarios introduces ambiguities, and the lack of identifiable texture on the surface limits the number of possible matches among images. Overall, \cref{eq:Mx0} is severely under-constrained in practice, with around a third of the singular values of $\mM^{(t)}$ being very close to zero \cite{ambiguities}. That means there are several possible solutions, in this case, obtained by moving each point along the line of sight. To factor out these incorrect solutions, we require additional constraints.

Coltraro \etal \cite{inextensible} used the Riemannian metric preservation assumption to constrain the possible vertices of a mesh for the simulation of deformable surfaces. In particular, they introduced an easy-to-evaluate function $C$ that, given a parametrization $\phi^{(t)}$ of $\cM_{\cS}^{(t)}$, satisfies
\begin{align}
	\mJ_{\phi^{(t)}}^ T\mJ_{\phi^{(t)}}=\mJ_{\phi^{\text{temp.}}}^ T\mJ_{\phi^{\text{temp.}}} \Longleftrightarrow C(\xx^{(t)})=\0\,,
\end{align}
where $\phi^{\text{temp.}}$ is the parametrization of the template shape. Given the realistic simulations achieved by Coltraro \etal \cite{inextensible}, we consider the assumption to be reasonable.

As done in some works \cite{laplacian_mesh,convex_optimization}, we can relax \cref{eq:Mx0} and minimize the error in the square sense subject to the metric preservation given by $C$. That is, for $\Delta \xx^{(t)}:=\xx^{(t)}-\xx^{(t-1)}$, the vertices of the mesh can be found solving
\begin{align}
	\begin{aligned}
		\min_{\Delta \xx^{(t)}}& \norm{\mM^{(t)}(\xx^{(t-1)}+\Delta \xx^{(t)})}^2_2 \\
		\text{s.t.}&\mC(\xx^{(t-1)}+\Delta \xx^{(t)})=\0
	\end{aligned}\,,
	\label{eq:hard_problem}
\end{align}
which is a quadratic program with quadratic constraints. As done by Coltraro \etal \cite{inextensible}, to make the problem computationally tractable, we approximate \cref{eq:hard_problem} with a sequence of quadratic programs with linear constraints using the first order Taylor expansion $\mC(\xx^{(t)})\approx \mC(\xx^{(t-1)})+\nabla \mC(\xx^{(t-1)}) \Delta\xx^{(t)}$ \cite{inextensible,constrained_lvm}.

Overall, for each time step $t$, the vertex positions of \textsc{Classical} are found by solving the linear system
\begin{align}
	\begin{cases}
		{\mM^{(t)}}^T\mM^{(t)}(\xx^{(t-1)}+\Delta \xx^{(t)})+\nabla \mC(\xx^{(t-1)})^T \boldsymbol{\lambda} =\0 \\
		\mC(\xx^{(t-1)}) + \nabla \mC(\xx^{(t-1)}) \Delta\xx^{(t)} = \0
	\end{cases}\,,
	\label{eq:linear_system_relaxed}
\end{align}
where $\boldsymbol{\lambda}$ is the vector of Lagrange multipliers. We recall that this loss function is not minimized during an offline training process but instead during inference for each of the inputs.

\subsection{Proposed method}
\label{sec:our_method}

\begin{algorithm*}[t]
	\caption{\textsc{Deformable Surface Reconstruction}}
	\label{alg:algo}
	\begin{algorithmic}[1]
		\STATE \textbf{Inputs: }Template $\cT$, matchings $\{M_{\cP}^{(t)}\}_{t\in [T]}$
		\STATE \textbf{Output: }Estimated surfaces described by $\{\phi_{\theta^{(t)}}\}_{t\in [T]}$
		\STATE Compute $P_{\cV}$ according to $\cT$
		\STATE $\displaystyle\theta^{\text{temp.}}\leftarrow \argmin_{\theta^{\text{temp.}}} \frac{1}{|P_{\cV}|} \sum_{\pp_i\in P_{\cV}} \norm{\phi_{\theta^{\text{temp.}}}(\pp_i) - \vv_i}_2$ \COMMENT{Over-fit to template}
		\STATE Store $\mJ_{\phi_{\theta^{\text{temp.}}}}(\pp)^ T\mJ_{\phi_{\theta^{\text{temp.}}}}(\pp)$ $\forall \pp \in P_{\cV}$ \COMMENT{Metric tensors of the template}
		\STATE Let $\phi_{\theta^{(0)}}\leftarrow\phi_{\theta^{\text{temp.}}}$ \COMMENT{Needed for \cref{eq:temp}}
		\FOR{$t\in [T]$}
		\STATE $\displaystyle\theta^{(t)}\leftarrow \argmin_{\theta^{(t)}} \cL_{\text{total}}$ \COMMENT{Compute loss \cref{eq:loss} with stored metric tensors}
		\ENDFOR
	\end{algorithmic}
\end{algorithm*}

Representing the state of a deformable object alone is an open challenge \cite{review_kragic,MONTAGNAT20011023}. Most \gls*{sft} methods represent objects using meshes, but an alternative is to use the explicit neural representation introduced by Groueix \etal \cite{atlasnet}. Such a representation encodes a surface in the weights of a neural network representing a mapping from 2D to 3D. Parametric representations are a general way to express a surface, and their great flexibility allows to control deformations with intuitive parameters \cite{review}. Given a parametric surface, we can analytically compute the first and second fundamental forms, Gaussian curvature, and surface normals \cite{parametric_surface_metric}.

Implicit neural representations are another method that has recently emerged as a promising alternative to classical discretized representations of signals \cite{inr_dictionaries2022}. Both explicit and implicit neural representations can generate continuous surfaces, which amounts to having meshes with an arbitrary resolution. However, an advantage of using explicit representations in front of the popular implicit representations such as NeRFS \cite{nerf} or SDFs \cite{sdf} is that the co-domain of an explicit representation is the surface itself. Therefore, if we want to generate a mesh, it is enough to sample the domain at desired locations of the vertices. Instead, implicit representations require using marching cubes on the outputs obtained by sampling many more points on $\mathbb{R}^3$ in the case of SDFs (and the sampled points will only lie on the surface if the output is 0), and sampling rays for NeRFs (with rays that may or may not hit the object).

Given that both considered datasets only consist of rectangular surfaces, we choose the parametrization $\cP$ to be $\cP:=[0,1]$ for practical purposes. Other approaches \cite{parametric_surface_metric,atlasnet} set $\cP$ as the interior of the unit square. However, to naturally represent the surface edges, we consider the closure of such a domain. The choice of $\cP$ to be the unit square allows representing all developable surfaces, \ie, smooth surfaces that we can flatten into a plane, \eg, cylinders, cones, and toruses.

To model non-developable surfaces like the sphere, we can choose the interior of the unit square. In case of having surfaces with other topologies or with holes, we could obtain $\cP$ by conformal flattening of $\cT$ \cite{shape_from_template}, using a texture map \cite{analytical_sol_proofs}, inferring the parametrization domain from data \cite{pix2surf_2020}, or combining different parametric surfaces to create an atlas \cite{atlasnet}.

A usual assumption in the \gls*{sft} problem is that the template $\cT$ and the tracked surface at time $t$ $\cS^{(t)}$ have the same topology, which implies that they also share a parametrization space \cite{shape_from_template}. Therefore, we can use the same choice of $\cP$ to infer the surface at any point in the sequence.

\begin{proposition}
    Let $\cS$ be a surface that can be parametrized on the unit square. There exists a feed-forward neural network with Softplus non-linearities that can approximate $\cS$ with arbitrary precision.
	\label{prop:rep_power}
\end{proposition}
\begin{proof}
	The proof follows the same reasoning as in Groueix \etal \cite[Proposition 2.]{atlasnet}, which states the same for ReLU non-linearities. Therefore, instead of relying on the universal representation theorem by Hornik \cite{universal_approximator_relu}, this proposition requires invoking the theorem by Kidger \etal \cite{universal_approximator}, which works with other activation functions, including the Softplus.
\end{proof}

Assuming that we can parametrize the surface of interest with the chosen $\cP$, we can leverage \cref{prop:rep_power} and represent it with a feed-forward neural network $\phi_{\theta}$, where $\theta$ are the learnable parameters. Note that the parametrization needs to be twice differentiable so that the loss may incorporate first-order derivatives \cite{parametric_surface_metric}. The requirement motivates the use of Softplus \cite{softplus}, an approximation of the ReLU function with smooth first and second derivatives \cite{parametric_surface_metric}.

Similarly to Brunet \etal \cite{brunet2010}, we obtain the surface parameters by minimizing a combination of a re-projection loss \cref{eq:proj}, a term involving the metric tensor favoring plausible poses \cref{eq:prior}, and a term that encourages smooth motions \cref{eq:temp}. The key differences with this work are:
\begin{itemize}
	\item \textbf{Re-projection loss: }For a matching $(\ss,\ii)\in M^{(t)}$, Brunet \etal \cite{brunet2010} enforce that $\ss \approx \Pi^{-1}(\ii)$. This requires knowing the depth, which is included as an optimized parameter. Instead, we check that $\Pi(\ss)\approx \ii$.
	\item \textbf{Metric preservation loss: }Our method enforces that the metric of the surface at any time is close to that of the template. In contrast, Brunet \etal \cite{brunet2010} sets the metric tensor to the identity, which only allows modeling unit squares on $\R^3$ when $\cP=[0,1]$.
	\item \textbf{Smoothing loss: }Brunet \etal \cite{brunet2010} favor non-bending surfaces by minimizing the Frobenius norm of the Hessian of the parametrization. Instead, we consider the difference of surfaces in consecutive frames, which aims at reducing frame-to-frame flickering.
\end{itemize}

For the sake of notation, let $\phi_{\star}$ be the real parametrization of $\cS$. Let $P_{\cV}$ be the set of points from $\cP$ corresponding to the vertices of the template mesh and the surface meshes for each time. That is
\begin{align}
	P_{\cV}:=\{\pp \in \cP: \pp=\phi_{\star}^{-1}(\vv) \ \forall \vv \in \cV \}\,.
	\label{eq:control_points}
\end{align}

Similarly to the relaxation of \cref{eq:matchings} used in \cref{sec:classical}, re-projection consistency can be enforced with
\begin{align}
	\cL_{\text{projection}}:=\frac{1}{|M^{(t)}_\cP|}\sum_{(\pp, \ii) \in M^{(t)}_\cP}\norm{\Pi(\phi_{\theta^{(t)}}(\pp))-\ii}_2\,,
	\label{eq:proj}
\end{align}
where
\begin{align}
	M_{\cP}^{(t)} := \{(\pp, \ii) : \pp \in \cP, \ii = \Pi(\phi_{\star}^{(t)}(\pp)) \in I^{(t)}\}\,.
	\label{eq:matchings_p}
\end{align}

This term enforces that the recovered surfaces are consistent with their corresponding images. This loss is present in all the \gls*{sft} solutions, either explicitly or implicitly in some neural network-based approaches, since the monocular images provide the only information to recover the current state of the surfaces.

Suppose the matches provide image locations for each vertex. One could displace each 3D vertex location along the line of sight, thus obtaining infinitely many surfaces that attain a zero re-projection loss. Enforcing isometry ideally reduces the set of plausible solutions to a single surface \cite[Theorem 1]{analytical_sol_proofs}.

To favor surfaces whose Riemannian metric is preserved, we add the loss term
\begin{align}
	\cL_{\text{metric}}:=\frac{1}{|P_{\cV}|}\sum_{\pp \in P_{\cV}}\left\lVert\right.&\mJ_{\phi_{\theta^{(t)}}}(\pp)^ T\mJ_{\phi_{\theta^{(t)}}}(\pp)-\\
	&\left.\mJ_{\phi_{\theta^{\text{temp.}}}}(\pp)^ T\mJ_{\phi_{\theta^{\text{temp.}}}}(\pp)\right\rVert_F^2\,,
	\label{eq:prior}
\end{align}
where $\norm{\cdot}_F$ is the Frobenius norm.

Note that unlike works approximating the surface metric, we can compute it analytically using surface parametrization \cite{parametric_surface_metric}. Moreover, the preservation of the metric is assessed in $P_{\cV}$, not only in the visible parts, which potentially helps to recover occluded zones \cite{brunet2010}.

Another difference with previous works is that we incorporate isometry as a soft constraint, which differs from works imposing the metric to be exactly preserved \cite{analytical_sol_proofs,shape_from_template}, which may be a restrictive assumption in real scenarios. Quasi-isometry, on the other hand, is a relatively mild constraint when manipulating surfaces like clothes, as those are nearly inextensible \cite{linear_local_model}. This approximation is especially suited for robotics contexts, where ﬁne details such as wrinkles are not needed \cite{inextensible}.

Finally, the loss also includes a temporal regularization term
\begin{align}
	\cL_{\text{time}}:=\frac{1}{|P_{\cV}|}\sum_{\pp \in P_{\cV}}\norm{\phi_{\theta^{(t)}}(\pp)-\phi_{\theta^{(t-1)}}(\pp)}_2\,.
	\label{eq:temp}
\end{align}

Adding a small amount of temporal regularization reduces frame-to-frame flickering \cite{temporal_coherence}. Additionally, point correspondences and metric preservation are not sufficient to uniquely recover the correct surface \cite{brunet2010}. In particular, surface corners can freely bend as long as they do not shrink or extend if there are no point correspondences \cite{brunet2010}.

The total loss used to update the parameters $\theta^{(t)}$ then becomes
\begin{align}
    \cL_{\text{total}}:=\cL_{\text{projection}}+\lambda_{\text{metric}}\cL_{\text{metric}}+\lambda_{\text{time}}\cL_{\text{time}}\,.
	\label{eq:loss}
\end{align}

In \cref{alg:algo}, we detail the procedure on how to estimate the surfaces for the whole sequence given the inputs of the \gls*{sft} problem.

\section{Experiments}
\label{sec:experiments}

\begin{table*}[t]
	\centering
	\rowcolors{2}{gray!10}{}
	\resizebox{%
      \ifdim\width>\linewidth
        \linewidth
      \else
        \width
      \fi
    }{!}{%
	\begin{tabular}{ccccccccc}
		\toprule
		\multicolumn{2}{c}{\textsc{Dataset} $\downarrow$ / \textsc{Methods}$\rightarrow$} &\textsc{GP} & \textsc{Lap} & \textsc{Dense} & \textsc{Graph} & \textsc{Texless} & \textsc{Classical} & \textsc{Ours}\\
		\toprule
		& Brick & 109.45 & \textbf{44.80} & 81.56 & 87.03 & 84.99\textsuperscript{$\dagger$} & 67.25 & \underline{48.32} \\
		& Campus & 85.48 & 85.80 & \underline{57.00} & 76.66 & 155.16 & 66.00 & \textbf{40.18}\\
		& Cloth & 104.30 & 460.44  & \underline{85.58} &   95.08    & 150.53 & 891.65 & \textbf{54.49}\\
		& Cobble &100.54  &\textbf{41.90}  &68.90   &66.23   &88.94 & 56.33 & \underline{53.05} \\
		& Cushion I & 104.17      &\textbf{72.55}   &108.65  &124.23  &100.22\textsuperscript{$\dagger$} & 252.36 & \underline{88.97} \\
		& Cushion II& 103.95  &157.76  & \underline{96.91}   &147.95  &226.35\textsuperscript{$\ddagger$} & 918.70 & \textbf{71.23} \\
		& Newspaper I &83.25 & \underline{63.23}   &85.12   &97.55   &130.26\textsuperscript{$\dagger$} & 79.61 & \textbf{43.23}\\
		& Newspaper II &79.88  & \underline{67.66}   &73.65   &87.33   &186.14 & 68.39 & \textbf{46.37}\\
		& Scene &102.25 &\textbf{57.36}   &82.37   &70.74   &136.24 & 69.85 & \underline{60.14} \\
		& Stone &100.62 &421.36  & \underline{90.93}   &91.28   &126.62\textsuperscript{$\dagger$} & 1140.05 & \textbf{72.03} \\
		\multirow{-11}{*}{\rotatebox[origin=c]{90}{\gls*{desurt}}} & Sunset &107.89 & 248.29\textsuperscript{$\ast$} & \underline{67.24}  &84.43   &138.52\textsuperscript{$\ddagger$} & 72.97 & \textbf{58.77} \\
		\midrule
		& Lr\_bottom\_edge & 0.10   & 0.88\textsuperscript{$\ast$}    &0.09    &1.09    &0.09\textsuperscript{$\dagger$} & \underline{0.07} & \textbf{0.06}\\
		& Lr\_bottom\_edge\_tl\_corn  & 0.06    & 3.00\textsuperscript{$\ast$}    &0.07    & 1.08      &0.06\textsuperscript{$\dagger$} & \textbf{0.04} & \underline{0.05} \\
		& Lr\_left\_edge & 0.08    & 1.60\textsuperscript{$\ast$}    &0.10    & 1.05      & 0.08\textsuperscript{$\dagger$} & \underline{0.06} & \textbf{0.05} \\
		& Lr\_tl\_tr\_corns   & 0.07  & 0.85\textsuperscript{$\ast$}    &0.09    & 1.00      &0.06\textsuperscript{$\dagger$} & \underline{0.06} & \textbf{0.05}\\
		& Lr\_top\_edge\_1 & 0.09    & 0.51\textsuperscript{$\ast$}    &0.09    & 0.96      & \underline{0.08}\textsuperscript{$\dagger$} & \textbf{0.07} & \textbf{0.07}\\
		& Lr\_top\_edge\_2 & 0.08    & 0.60\textsuperscript{$\ast$}    &0.09    & 0.97     & \underline{0.07}\textsuperscript{$\dagger$} & \textbf{0.06} & \underline{0.07}\\
		\multirow{-7}{*}{\rotatebox[origin=c]{90}{TDS}} & Lr\_top\_edge\_3 & \underline{0.06}    & 0.36\textsuperscript{$\ast$}    &0.07    & 1.02      &0.07\textsuperscript{$\ddagger$} & \textbf{0.05} & \underline{0.06}\\
		\bottomrule
	\end{tabular}
	}%
	\caption{\textbf{Quantitative results}. This table shows the mean tracking error (mm) obtained when reconstructing a mesh from monocular images. The \textbf{best value} for each sequence (the lowest) is shown in boldface and the \underline{runner-up} is underlined. Asterisks ($\ast$) indicate that a method did not terminate, in which case the average error of the meshes obtained before the method crashed is reported. The sequences used to train and validate \textsc{Texless} are indicated with a dagger ($\dagger$) and a double dagger ($\ddagger$), respectively. Note that we favor the algorithm by providing it with full sequences and in some cases reporting the values on training examples.}
	\label{tab:quantitative}
\end{table*}
\begin{table}[t]
    \resizebox{%
      \ifdim\width>\columnwidth
        \columnwidth
      \else
        \width
      \fi
    }{!}{%
	\centering
	\rowcolors{2}{gray!10}{}
	\begin{tabular}{llccc}
		\toprule
		\multicolumn{2}{l}{\textsc{Methods}}&\textsc{Time }(s) & \textsc{\gls*{desurt}} & \textsc{\gls*{tds}}\\
		\toprule
		& \textsc{GP} &  $\mathbf{0.01}$ & $98.34\pm 10.37$ & $0.07 \pm \mathbf{0.01}$ \\
		\multirow{-2}{*}{Learning-based}& \textsc{Texless} & $\mathbf{0.01}$ & $138.54 \pm 41.72$ & $0.07 \pm \mathbf{0.01}$ \\
		\midrule
		& \textsc{Lap} &  $5.80$ & $156.47 \pm 153.31$ & $1.11 \pm 0.92$ \\
		& \textsc{Dense} & $26.73$ & $81.63 \pm \mathbf{14.55}$ & $0.09 \pm \mathbf{0.01}$ \\
		& \textsc{Graph} & $32.52$ &  $93.50 \pm 23.75$ & $1.02 \pm 0.05$ \\
		& \textsc{Classical} & $49.00$ & $334.83 \pm 424.56$ & $\mathbf{0.06 \pm 0.01}$ \\
		\multirow{-5}{*}{Optimization-based} & \textsc{Ours} & $\mathbf{0.95}$ &  $\mathbf{57.80} \pm 14.72$ & $\mathbf{0.06 \pm 0.01}$ \\
		\bottomrule
	\end{tabular}
        }
	\caption{\textbf{Additional quantitative performance indicators.} The column \textsc{Time} shows the average the number of seconds required to process one image for the \gls*{tds} sequence Lr\_bottom\_edge, the dataset with higher resolution meshes  (concretely, 961 vertices). The columns \textsc{\gls*{desurt}} and \textsc{\gls*{tds}} present the mean and standard deviation of the tracking error in millimeters across the sequences of each dataset. The \textbf{best value} for each sequence (the lowest for all metrics) is shown in boldface.}
	\label{tab:additional_quantitative}
\end{table}

\subsection{Datasets}
\label{sec:datasets}

To quantitatively evaluate the proposed method, we require ground-truth mesh vertex positions. In this paper, we use two public datasets involving rectangular deformable surfaces:

\begin{itemize}
	\item \textbf{\gls*{desurt} \cite{surface_track_graph_matching}:} Dataset consisting of 11 video streams, with around 300 images each, displaying different types of materials, deformations, and lighting conditions. The surfaces are either well-textured (Campus, Cobble, Cushion I, Scene, Newspaper I, and Newspaper II), repetitively textured (Brick, Cloth, and Cushion II), or weakly textured (Stone and Sunset). The \gls*{desurt} dataset represents surfaces with meshes of size $13\times10$ vertices for all the sequences but the cushion, which uses an $11\times11$ mesh.

	\item \textbf{\gls*{tds} \cite{texture-less}:} Dataset showing deformable surfaces under various lighting conditions. We use all the sequences with ground-truth vertex annotations, which account for seven image sequences with around 900 images each, displaying a piece of cloth represented with a $31\times31$ mesh.
\end{itemize}

\begin{figure*}
    \centering
    \begin{subfigure}{.24\textwidth}
      \centering
      \includegraphics[width=\textwidth]{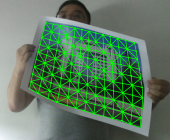}
      \caption{\textsc{Ground truth}}
    \end{subfigure}%
    \hfill
    \begin{subfigure}{.24\textwidth}
      \centering
      \includegraphics[width=\textwidth]{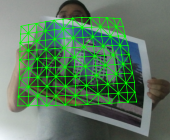}
      \caption{\textsc{GP}}
    \end{subfigure}%
    \hfill
    \begin{subfigure}{.24\textwidth}
      \centering
      \includegraphics[width=\textwidth]{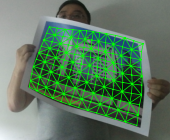}
      \caption{\textsc{Lap}}
    \end{subfigure}%
    \hfill
    \begin{subfigure}{.24\textwidth}
      \centering
      \includegraphics[width=\textwidth]{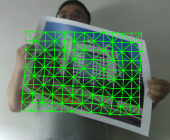}
      \caption{\textsc{Dense}}
    \end{subfigure}%
    \hfill
    \begin{subfigure}{.24\textwidth}
      \centering
      \includegraphics[width=\textwidth]{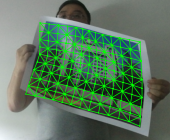}
      \caption{\textsc{Graph}}
    \end{subfigure}%
    \hfill
    \begin{subfigure}{.24\textwidth}
      \centering
      \includegraphics[width=\textwidth]{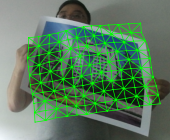}
      \caption{\textsc{Texless}}
    \end{subfigure}%
    \hfill
    \begin{subfigure}{.24\textwidth}
      \centering
      \includegraphics[width=\textwidth]{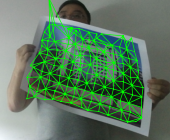}
      \caption{\textsc{Classical}}
    \end{subfigure}%
    \hfill
    \begin{subfigure}{.24\textwidth}
      \centering
      \includegraphics[width=\textwidth]{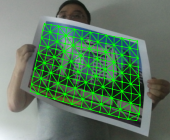}
      \caption{\textsc{Ours}}
    \end{subfigure}%
    \caption{\textbf{Qualitative results}. Comparison of the reconstructions obtained by different methods on the frame 85 of the campus sequence of \gls*{desurt} \cite{surface_track_graph_matching}.}
    \label{fig:qualitative}
\end{figure*}

\subsection{Baselines}

We use the following methods to compare the performance of the proposed technique. Unless specified, we use the publicly available code implemented by the original authors without modifying the algorithms or the hyper-parameters.
\begin{itemize}
	\item \textbf{GP: }Unconstrained \gls*{gplvm} \cite{constrained_lvm}. We exclude the constrained GP-LVM introduced in the same paper from the comparisons as it consistently underperformed the unconstrained version.
	\item \textbf{Lap: }Method based on Laplacian meshes \cite{lap_code_2,laplacian_mesh,lap_code_1}.
	\item \textbf{Dense: }Dense image registration algorithm by Ngo \etal \cite{dense_image_registration}.
	\item \textbf{Graph: }Deformable surface tracking using graph matching \cite{surface_track_graph_matching}.
	\item \textbf{TexLess: }The model from Bednar{\'{\i}}k \etal \cite{texture-less} trained from scratch with ground-truth mesh vertices for each dataset and mesh resolution.
	\item \textbf{Classical: }Classical method outlined in \cref{sec:classical}.
	\item \textbf{Ours: }Proposed method described in \cref{sec:our_method}.
\end{itemize}

\subsection{Implementation details}
\label{sec:implementation}

\textsc{Classical} requires solving a very sparse linear system, \ie, \cref{eq:linear_system_relaxed}, multiple times to approximate \cref{eq:hard_problem}. In practice, such a linear system is solved in the least-squares sense using the SciPy \cite{scipy} routine limited to 100 iterations multiple times until the approximation is good enough as understood by the same criterion of Coltraro \etal \cite{inextensible}.

All the surfaces of the datasets considered in this work are rectangular and represented using a mesh with equispaced vertices. For this reason, we can define $P_{\cV}$ with the coordinates given by a grid on the unit square using the width and height of the mesh. Concretely, we follow the convention to assign those values using the ordering of the vertices as seen on the template image (see \cref{fig:notation} for an illustration of the parametrization on the $RG$ color space).

The matches $M^{(t)}$ can be obtained with several methods such as SIFT \cite{sift}, SURF \cite{surf}, Ferns \cite{ferns}, or soft graph matching \cite{surface_track_graph_matching} and then applying an outlier rejection mechanism \cite{outlier_rejection_deformable}. Obtaining the matches with a dedicated model and then performing reconstruction results in correspondences robust to occlusions, especially for well-textured surfaces \cite{dense_image_registration}.

We obtain correspondences with the registration algorithm used in the baseline \textsc{Lap}, based on matching SIFT \cite{sift} features and then applying an outlier rejection mechanism. The \textsc{Lap} algorithm did not terminate and hence could not provide matchings for all the frames indicated with an asterisk ($\ast$) in \cref{tab:quantitative}. In this case, we used synthetic matches for simplicity by sampling a random point inside each facet of the mesh. To set the weight for each regularization term in \eqref{eq:loss}, we perform a grid search on the evaluation loss testing the values
\begin{align}
	\lambda_{\text{metric}},\lambda_{\text{time}}\in \{0, 10^{-2}, 10^{-1},\dots, 10^2\}
\end{align}
on the first sequence for each dataset. We use $(\lambda_{\text{metric}},\lambda_{\text{time}})=(0.01,0.001)$ on \gls*{desurt} and $(\lambda_{\text{metric}},\lambda_{\text{time}})=(100,100)$ on \gls*{tds}. We can justify the high regularization values on \gls*{tds} because the \gls*{tds} sequences show a cloth pinned to a fixed bar along a given edge or corners \cite{texture-less}, so the position of at least part of the cloth was quite stable, which is exploited by both the temporal and metric constraints. Finally, although we used synthetic matches in this case, the best results obtained by our method were achieved when the relative contribution of the projection loss was the lowest.

We represent the surface parametrization with a multi-layer perceptron with three hidden layers having 128, 256, and 128 units and implemented using the PyTorch \cite{pytorch} framework. We obtain the parameters of this model by minimizing \cref{eq:loss} using the ADAM optimizer \cite{adam}.

\subsection{Evaluation}

To evaluate the accuracy of the proposed model, we compute the Euclidean distances from vertex to vertex. In particular, \cref{tab:quantitative} reports the mean of such distances, \ie,
\begin{align}
	\frac{1}{T}\sum_{t\in[T]}\left[ \frac{1}{N}\sum_{n\in[N]} \norm{\hat{\vv}_n^{(t)}-\vv_n^{(t)}}_2 \right]\,,
	\label{eq:eval}
\end{align}
the quantity reported in several 3D reconstruction works \cite{dense_image_registration,pumarola2018geometry,constrained_lvm,surface_track_graph_matching}. Given a parametrization $\phi_{\theta^{(t)}}$ obtained with the proposed method, one can compute $\hat{\vv}_n^{(t)}$ by evaluating the parametrization at the point in $P_{\cV}$ corresponding to the $n-$th vertex.

It is worth noting that the best performance for all sequences is attained by methods relying on feature matches. In particular, \textsc{Lap} achieves the best performance for some sequences but fails on others, which is consistent with the results in Kairanda \etal \cite{phi_sft}. The superiority of algorithms taking matches as input contrasts with the current trend of directly predicting surfaces directly from images. However, relying on an external registration algorithm is a double-edged sword and becomes the main limitation of the proposed method. The reason is that matching algorithms introduce noise and fail with repetitive or poorly textured surfaces (\eg, \gls*{tds} dataset).

An alternative is to use dense approaches like \textsc{Dense}, which does not extract features but instead maximizes a similarity measure to perform registration \cite{dense_image_registration,temporal_coherence}. These approaches usually need consistent illumination and suffer from brightness changes, occlusions, and motion blur \cite{surface_track_graph_matching}.

Data-based approaches like \textsc{Texless} do not require matches but typically work only with the surface seen during training and require fine-tuning to different templates. Fuentes-Jimenez \etal \cite{texture_generic_deepsft} attempted training texture-generic neural networks, but their results are still less accurate than the ones obtained with texture-specific methods. \cref{tab:quantitative} showcases that data-based approaches did not perform strictly better than other methods for any tested sequences.

\cref{tab:additional_quantitative} reports the average time required by each method to process one image and statistics about the performance across all the sequences reported in \cref{tab:quantitative}. On the one hand, as expected, learning-based methods attain the lowest inference time as they only need to evaluate a function. Optimization-based methods require finding the best surface parameters given an image, which requires performing several function evaluations and parameter updates. The proposed method is optimization-based but has significantly lower computational overhead than its competitors. On the other hand, the proposed approach attains the best average performance across both tested datasets. Moreover, the standard deviation of mean tracking errors for different sequences is among the lowest, which shows that our method is generally applicable and robust.

In \cref{fig:qualitative}, we show a qualitative evaluation of the obtained results. We depict the projection of the reconstructed mesh on top of the input image for each tested method and the ground truth. Note that, as mentioned above, a perfect vertex projection does not guarantee that the reconstruction is correct due to the depth ambiguity. Therefore, one must consider both qualitative and quantitative results, the former considering how well the projected mesh matches the image and the latter considering the estimation error in 3D.

While \textsc{Classical} achieves one of the lowest mean tracking errors, some recovered surfaces are irregular. \textsc{Dense} attains the second-best reconstruction performance on the campus sequence despite failing to recover the surface in some individual frames. This approach does not rely on an external feature-matching algorithm and instead uses dense matching of image features in the optimized cost function. Although this is beneficial for poorly textured surfaces, it may lead to suboptimal re-projection constraints for identifiable textures. However, the quantitative metric being one of the best, shows that it better resolves the depth ambiguities than other methods.

The \textsc{Lap} and \textsc{Graph} methods also yield one of the best quantitative results on the campus sequence, this time reflected in qualitatively faithful surface reconstructions. Nonetheless, we can see mismatches with the surface contour \wrt to the \textsc{Ground truth}. \textsc{Ours} attains the best quantitative performance in \cref{tab:quantitative}, shows good qualitative results, and is better at recovering the surface contours than the alternatives.

\section{Conclusions}
\label{sec:conclusions}

This work tackles the inherently ill-posed problem of reconstructing deformable surfaces from monocular images. The proposed method assumes that the Riemannian metric of the manipulated surface is approximately constant or equivalently that the transformation from the template object is an isometry. Metric preservation consists of a mild hypothesis for a wide variety of surfaces since many materials do not perceptibly shrink or stretch when they suffer deformations \cite{reconstruct_sharply_folding}.

The Riemannian metric preservation constraint proposed by Coltraro \etal \cite{inextensible} for cloth simulation can easily be incorporated into the \gls*{sft} problem, which led to the approach denoted as \textsc{Classical}. The results obtained with this approach achieve one of the best performances for the \gls*{tds} sequences (see \cref{tab:quantitative}). Despite the notable results, which back up the metric preservation assumption, this approach scales poorly with the number of vertices and diverges for some sequences (see \eg, the Stone sequence in \cref{tab:quantitative}). For this reason, instead of naively incorporating the constraints and optimizing the vertex positions, we propose to use explicit neural surfaces.

Parametric surfaces learned by neural networks pose an attractive framework to represent surfaces with arbitrary precision, an observation formalized in \cref{prop:rep_power}. Having a continuous surface, we avoid discretization problems when estimating the surface parameters and can generate meshes with different levels of detail. The learned surface parametrization allows for the analytical computation of differential geometric quantities. In this work, we used the well-known isometry constraint, but we could easily modify the proposed framework to enforce other constraints, such as constant surface area as done by Salzmann \etal \cite{convex_optimization}.

Another advantage of using neural networks is that we can apply non-convex constraints involving partial derivatives and, therefore, not rely on the relaxations proposed by Coltraro \etal \cite{inextensible} needed in a classical optimization framework. Contrasting to the previous methods for \gls*{sft} using neural networks, our approach does not require offline training. Therefore, it lifts the requirement of a dataset having enough samples to represent the dynamics and appearance of an object, thus overcoming the problems of the data-based approaches described in \cref{sec:databased}. Among others, the consequences are that we can apply the proposed method to any sequence without modification and that it does not require a dataset to infer the manifold of plausible deformations from data. The solution to avoid training is to use an iterative optimization process for each input, but such optimization takes orders of magnitude less time than the alternative methods considered in this work.

\section*{Acknowledgments}

This work is part of the project CLOTHILDE (``CLOTH manIpulation Learning from DEmonstrations") which has received funding from the European Research Council (ERC) under the European Union’s Horizon 2020 research and innovation program (Advanced Grant agreement No. 741930). O.B. thanks Franco Coltraro and Xavier Gràcia for their insights on differential geometry, and the European Laboratory for Learning and Intelligent Systems (ELLIS) PhD program for support.

\bibliographystyle{ieee_fullname}
\bibliography{references}

\end{document}